\documentclass{article}

\usepackage{graphicx} 
\usepackage{subfigure} 

\usepackage{natbib}

\usepackage{algorithm}
\usepackage{algorithmic}

\usepackage{hyperref}

\usepackage[accepted]{icml2013}

\icmltitlerunning{Online Learning under Delayed Feedback}

\usepackage{amssymb}
\usepackage{url}
\usepackage{upgreek}
\usepackage{amsmath}
\usepackage{amsthm}
\usepackage{algorithm}
\usepackage{algorithmic}

\usepackage{nccmath}
    
\usepackage{float}
\floatstyle{boxed} 
\restylefloat{figure}

\usepackage{amssymb}
\usepackage{amsmath}
\usepackage{xspace}

\newcommand{\ExpVal}[1]{\mathbb{E}\left[#1\right]}
\newcommand{\EE}[1]{\mathbb{E}\left[ #1 \right]}
\newcommand{\Variance}[1]{\mathbb{\sigma}^2\left[ #1 \right]}
\newcommand{\Prob}[1]{\mathbb{P}\left\lbrace #1 \right\rbrace}
\newcommand{\Event}[1]{\mathbb{I}\left\lbrace #1 \right\rbrace}

\newcommand{\SimpleEvent}{\mathbb{I}}

\newcommand{\ceiling}[1]{\left\lceil #1 \right\rceil }
\newcommand{\floor}[1]{\left\lfloor #1 \right\rfloor }

\newcommand{\Naturals}{\mathbb{N}}

\newcommand{\Reals}{\mathbb{R}}

\newcommand{\ActionSet}{\mathcal{A}}
\newcommand{\OutcomeSet}{\mathcal{B}}
\newcommand{\RewardSet}{\mathcal{R}}
\newcommand{\XSet}{\mathcal{X}}
\newcommand{\FeedbackSet}{\mathcal{H}}

\newcommand{\kldiv}{d}

\newcommand{\Base}{\textsc{Base}\xspace}

\newcommand{\StochBlackBoxAlg}{QPM-D\xspace}

\newcommand{\thtimes}[1]{{#1}^{\text{th}}}

\newtheorem{theorem}{Theorem}

\newtheorem{lemma}[theorem]{Lemma}
\newtheorem{corollary}[theorem]{Corollary}

\newcommand{\eqdef}{\doteq}
\renewcommand{\algorithmicforall}{\textbf{for each}}

\begin{document} 

\twocolumn[
\icmltitle{Online Learning under Delayed Feedback}

\icmlauthor{Pooria Joulani}{pooria@ualberta.ca}
\icmlauthor{Andr\'as Gy\"orgy}{gyorgy@ualberta.ca}
\icmlauthor{Csaba Szepesv\'ari}{szepesva@ualberta.ca}
\icmladdress{Dept. of Computing Science, University of Alberta, 
            Edmonton, AB, T6G 2E8 CANADA}            

\icmlkeywords{online learning, partial information feedback, delayed feedback}

\vskip 0.3in
]

\begin{abstract}
Online learning with delayed feedback has received increasing attention recently due to its several applications in 
distributed, web-based learning problems. In this paper we provide a systematic study of the topic, and analyze the
effect of delay on the regret of online learning algorithms. 
Somewhat surprisingly, it turns out that delay increases the regret in a multiplicative way in adversarial problems, and in an additive 
way in stochastic problems. We give meta-algorithms that transform, in a black-box fashion, 
algorithms developed for the non-delayed case into ones that can handle the presence of delays in the feedback loop.
Modifications of the well-known UCB algorithm are also developed for the bandit problem with delayed feedback, with
the advantage over the meta-algorithms that they can be implemented with lower complexity.
\end{abstract}

\section{Introduction}
In this paper we study sequential learning when the feedback about the predictions made by the forecaster are delayed. This is the case, for example, in web advertisement, where the information whether a user has clicked on a certain ad may come back to the engine in a delayed fashion: after an ad is selected, while waiting for the information if the user clicks or not, the engine has to provide ads to other users. Also, the click information may be aggregated and then periodically sent to the module that decides about the ads, resulting in further delays.
\citep{Li:2010,Dudik:2011}. Another example is parallel, distributed learning, where propagating information among nodes causes delays \citep{Agarwal:2011}. 

While online learning has proved to be successful in many machine learning problems and is applied in practice in situations where the feedback is delayed, the theoretical results for the non-delayed setup are not applicable when delays are present. 
Previous work concerning the delayed setting focussed on specific online learning settings and delay models (mostly with constant delays). Thus, a comprehensive understanding of the effects of delays is missing. In this paper, we provide a systematic study of online learning problems with delayed feedback. We consider the \emph{partial monitoring setting}, which covers all settings previously considered in the literature, extending, unifying, and often improving upon existing results.
In particular, we give general meta-algorithms that transform, in a black-box fashion, algorithms developed for the non-delayed case into algorithms that can handle delays efficiently. We analyze how the delay effects the regret of the algorithms. One interesting, perhaps somewhat surprising, result is that the delay inflates the regret in a multiplicative way in adversarial problems, while this effect is only additive in stochastic problems. While our general meta-algorithms are  useful, their time- and space-complexity may be unnecessarily large. To resolve this problem, we work out modifications of variants of the UCB algorithm \citep{Auer:2002} for stochastic bandit problems with delayed feedback that have much smaller complexity than the black-box algorithms.

The rest of the paper is organized as follows. The problem of online learning with delayed feedback is defined in Section~\ref{sec:def}. The adversarial and stochastic problems are analyzed in Sections~\ref{sec:adversarial} and~\ref{sec:stochastic}, while the modification of the UCB algorithm is given in Section~\ref{sec:UCB}. Some proofs, as well as results about the KL-UCB algorithm \citep{Garivier:2011} under delayed feedback,  are provided in the appendix.
\section{The delayed feedback model}
\label{sec:def}
We consider a general model of online learning, which we call the partial monitoring problem with side information.
In this model, 
the forecaster (decision maker) has to make a sequence of predictions (actions), possibly based on some side information, and for each prediction it receives some reward and feedback, where the feedback is delayed. 
More formally, given a set of possible side information values $\XSet$, a set of possible predictions $\ActionSet$, a set of reward functions 
$\RewardSet \subset \{r : \XSet \times \ActionSet \to  \Reals\}$, and a set of possible feedback values $\FeedbackSet$, at each time instant 
$t=1,2,\ldots$, the forecaster receives some side information $x_t \in \XSet$; then, possibly based on the side information, the forecaster predicts some value $a_t \in \ActionSet$ while the environment simultaneously chooses a reward function $r_t \in \RewardSet$; finally, the forecaster receives reward $r_t(x_t,a_t)$ and some time-stamped feedback set $H_t \subset \mathbb{N} \times \FeedbackSet$. In particular, each element of $H_t$ is a pair of time index and a feedback value, the time index indicating the time instant whose decision the associated feedback corresponds to.

Note that the forecaster may or may not receive any direct information about the rewards it receives (i.e., the rewards may be hidden).
In standard online learning, the feedback-set $H_t$ is a singleton and the feedback in this set depends on $r_t,a_t$. In the delayed model, however, the feedback that concerns the decision at time $t$ is received at the end of the time period $t+\tau_t$, \emph{after} the prediction is made, i.e., it is delayed by $\tau_t$ time steps. Note that $\tau_t\equiv 0$ corresponds to the non-delayed case. Due to the delays multiple feedbacks may arrive at the same time, hence the definition of $H_t$.

\begin{figure}
{\bf Parameters:}
Forecaster's prediction set $\ActionSet$, set of outcomes $\OutcomeSet$, side information set $\XSet$, reward function $r: \XSet \times \ActionSet \times \OutcomeSet \to \Reals$, feedback function $h: \XSet \times \ActionSet \times \OutcomeSet \to \FeedbackSet$, time horizon $n$ (optional).\\
At each time instant $t = 1,2,\dots,n$:
\begin{enumerate}
\item The environment chooses some side information  $x_t \in \XSet$ and an outcome $b_t \in \OutcomeSet$.
\item The side information $x_t$ is presented to the forecaster, who makes a prediction $a_t \in \ActionSet$, which results in the reward $r(x_t,a_t,b_t)$ (unknown to the forecaster).
\item The feedback $h_t = h(x_t,a_t,b_t)$ is scheduled to be revealed after $\tau_t$ time instants.
\item The agent observes $H_t = \{(t',h_{t'}) : t' \le t, t' + \tau_{t'} = t\}$, i.e., all the feedback values scheduled to be revealed at time step $t$, together with their timestamps.
\end{enumerate}
\caption{Partial monitoring under delayed, time\-stamped feedback.}
\label{fig:DelayedOnlineLearningProtocol}
\end{figure}

The goal of the forecaster is to maximize its cumulative reward $\sum_{t=1}^n r_t(x_t,a_t)$ $(n\ge 1)$. 
The performance of the forecaster is measured relative to the best static strategy selected from some set $\mathcal{F} \subset \{ f\,|\, f: \XSet \to \ActionSet\}$ in hindsight. In particular, the forecaster's performance is measured through the \emph{regret},  defined by 
\[
R_n= \sup_{a\in \mathcal{F}} \sum_{t=1}^n r_t(x_t,a(x_t)) - \sum_{t=1}^n r_t(x_t, a_t).
\]
A forecaster is consistent if it achieves, asymptotically, the average reward of the best static strategy, that is $\EE{R_n}/n \to 0$, and we are interested in how fast the average regret can be made to converge to $0$.

The above general problem formulation includes most scenarios considered in online learning. 
In the full information case, the feedback is the reward function itself, that is, $\mathcal{H} = \mathcal{R}$ and
$H_t = \{(t,r_t)\})$ (in the non-delayed case).
In the bandit case, the forecaster only learns the rewards of its own prediction, i.e., $\mathcal{H} = \mathbb{R}$ and $H_t = \{(t,r_t(x_t,a_t))\}$. 
In the partial monitoring case, the forecaster is given a reward function
$r:\XSet\times\ActionSet\times\OutcomeSet \to \Reals$ and a feedback function $h:\XSet\times\ActionSet\times\OutcomeSet \to \FeedbackSet$, where $\OutcomeSet$ is a set of choices 
(outcomes) of the environment. 
Then, for each time instant the environment picks an outcome $b_t \in \OutcomeSet$, 
	and the reward becomes
	$r_t(x_t,a_t)=r(x_t,a_t,b_t)$, while $H_t = \{(t,h(x_t,a_t,b_t))\}$.
This interaction protocol is shown in Figure~\ref{fig:DelayedOnlineLearningProtocol} in the delayed case. Note that the bandit and full information problems can also be treated as special partial monitoring problems. Therefore, we will use this last formulation of the problem. When no stochastic assumption is made on how the sequence $b_t$ is generated, we talk about the adversarial model. In the stochastic setting we will consider the case when $b_t$ is a sequence of independent, identically distributed (i.i.d.) random variables. Side information may or may not be present in a real problem; in its absence $\XSet$ is a singleton set. 

Finally, we may have different assumptions on the delays. 
Most often, we will assume that $(\tau_t)_{t\ge 1}$ is an i.i.d. sequence, which is independent of the past predictions $(a_s)_{s \le t}$ of the forecaster. In the stochastic setting, we also allow the distribution of $\tau_t$ to depend on $a_t$.

Note that  the delays may change the order of observing the feedbacks, with the feedback of a more recent prediction being observed before the feedback of an earlier one.
\subsection{Related work}
The effect of delayed feedback has been studied in the recent years under different online learning scenarios and different assumptions on the delay.
A concise summary, together with the contributions of this paper, is given in Table~\ref{tbl:Contrib}.
\begin{table*}[ht]
\begin{center}
			\begin{tabular}{|c|c|c|c|}
			\hline
				& & { Stochastic Feedback} & { General (Adversarial) Feedback} \\
			\hline
				 & { } &  & \multicolumn{1}{l|}{{  L}}\\
				 & {No}& {$R(n) \le R'(n) +O(\ExpVal{\tau_t^2})$} & {$R(n) \le O(\tau_{const}) \times R'(n/\tau_{const}) $}\\ 
				& { Side}  & \multicolumn{1}{l|}{{ \cite{Agarwal:2011}}} & \multicolumn{1}{l|}{{\cite{Weinberger:2002}}} \\
			{ Full Info}  & {Info} & & \multicolumn{1}{l|}{{\cite{Langford:2009}}}\\
		 	 & &  & \multicolumn{1}{l|}{{\cite{Agarwal:2011}}}\\ \cline{2-4} 
				 & &\multicolumn{1}{l|}{{  L}} &\multicolumn{1}{l|}{{  L}} \\
			& { Side Info} & {$R(n) \le R'(n) + O(D^*)$} &{ $ R(n) \le O(\bar{D}) \times R'(n/\bar{D})$}\\
				 & { } &\multicolumn{1}{l|}{ { \cite{Mesterharm:2007} } }&\multicolumn{1}{l|}{{ \cite{ Mesterharm:2007}}}\\ \cline{2-4}
			\hline
				& { No Side }& { $R(n) \le C_1 R'(n) + C_2 \tau_{\max} \log(\tau_{\max})$ } & { $R(n) \le O(\tau_{const}) \times R(n/\tau_{const})$} \\
				 { Bandit} & { Info} & \multicolumn{1}{l|}{ \cite{Desautels:2012}} & \multicolumn{1}{l|}{\citep{Neu:2010}} \\   \cline{2-4} 
			{ Feedback}	& & & \\
			 & { Side Info} & {$R(n) \le R'(n) + O(\tau_{const} \sqrt{\log n})$} & \\
				 & { } & \multicolumn{1}{l|}{{ \cite{Dudik:2011} }} & \\ \hline
			{Partial} & \begin{tabular}[m]{c} No \\ Side Info \end{tabular} & $\mathbf{R_n \le  R'(n) + O(G^\ast_n)}$ &  $\mathbf{R_n \le \left(1+\ExpVal{G^\ast_n} \right)\times R'\left(\dfrac{n}{1+\ExpVal{G^\ast_n}} \right)}$ \\ \cline{2-4} 
			Monitoring & Side Info &  & $\mathbf{R_n \le \left(1+\ExpVal{G^\ast_n} \right)\times R'\left(\dfrac{n}{1+\ExpVal{G^\ast_n}} \right)}$
			\\ \hline
			\end{tabular}
			\caption{Summary of work on online learning under delayed feedback. $R(n)$ shows the (expected) regret in the delayed setting, while $R'(n)$ shows the (upper bound on) the (expected) regret in the non-delayed setting. $L$ denotes a matching lower bound. $D^*$ and $\bar{D}$ indicate the maximum and average \emph{gap}, respectively, where a gap is a number of consecutive time steps the agent does not get any feedback (in the adversarial delay formulation used by \citet{Mesterharm:2005,Mesterharm:2007}). The term $\tau_{const}$ indicates that the results are for constant delays only. For the work of \cite{Desautels:2012},  $C_1$ and $C_2$ are positive constants, with $C_1 > 1$, and $\tau_{\max}$ denotes the maximum delay. The results presented in this paper are shown in boldface, where 
	$\mathbf{G^\ast_t}$ is the maximum number of outstanding feedbacks during the first $t$ time-steps.
In particular,			
			$\mathbf{G^\ast_{n} }\le \tau_{max}$ when the delays have an upper bound $\tau_{max}$, and we show that $\mathbf{G^\ast_{n} = O\left(\ExpVal{\tau_t} + \sqrt{\ExpVal{\tau_t}\log n} + \log n \right)}$ when the delays $\tau_t$ are i.i.d. The new bounds for the partial monitoring problem are automatically applicable in the other, spacial, cases, and give improved results in most cases.}
			\label{tbl:Contrib}
\end{center}
		\end{table*}

To the best of our knowledge, \citet{Weinberger:2002} were the first to analyze the delayed feedback problem; they considered the adversarial 
full information setting with a fixed, known delay $\tau_{const}$. They showed that the minimax optimal solution is to run $\tau_{const}+1$ independent optimal predictors on 
the subsampled reward sequences: $\tau_{const}+1$ prediction strategies are used such that the $\thtimes{i}$ predictor is used at time instants $t$ with $\left( t \mod (\tau_{const}+1) \right)+1=i$.
This approach forms the basis of our method devised for the adversarial case (see Section~\ref{sec:adversarial}). \citet{Langford:2009} showed that under the usual conditions, 
a sufficiently slowed-down version of the mirror descent algorithm achieves optimal decay rate of the average regret.
~\citet{Mesterharm:2005,Mesterharm:2007} considered another variant of the full information setting, using an adversarial model on the delays in the label prediction setting, where the forecaster has to predict the label corresponding to a side information vector $x_t$. While in the full information online prediction problem \citet{Weinberger:2002}
showed that the regret increases by a multiplicative factor of $\tau_{const}$, in the work of \citet{Mesterharm:2005,Mesterharm:2007} the important quantity becomes the
maximum/average gap defined as the length of the largest time interval the forecaster does not receive feedback. \citet{Mesterharm:2005,Mesterharm:2007}  also
shows that the minimax regret in the adversarial case increases multiplicatively by the average gap, while it increases only in an additive fashion in the stochastic case, by the maximum gap.
~\citet{Agarwal:2011} considered the problem of online stochastic optimization and showed that, for i.i.d. random delays, the regret increases with an additive factor of order $\ExpVal{\tau^2}$.

Qualitatively similar results were obtained in the bandit setting. Considering a fixed and known delay $\tau_{const}$, 
~\citet{Dudik:2011} showed an additive $O(\tau_{const} \sqrt{\log n})$ penalty in the regret for the stochastic setting (with side information), while
~\cite{Neu:2010} showed a multiplicative regret for the adversarial bandit case. The problem of delayed feedback has also been studied for Gaussian process bandit optimization \citep{Desautels:2012}, resulting in a multiplicative increase in the regret that is independent of the delay and an additive term depending on the maximum delay.

In the rest of the paper we generalize the above results to the partial monitoring setting, extending, unifying, and often improving existing results.

\section{Black-Box Algorithms for Delayed Feedback}
\label{sec:balckbox}
In this section we provide black-box algorithms for the delayed feedback problem. We assume that there exists a base algorithm \Base for solving the prediction problem without delay.
We often do not specify the assumptions underlying the regret bounds of these algorithms, and assume that the problem we consider only differs from the original problem because of the delays. For example, in the adversarial setting, \Base\ may build on the assumption that the reward functions are selected in an oblivious or non-oblivious way (i.e., independently or not of the predictions of the forecaster). First we consider the adversarial case in Section~\ref{sec:adversarial}. Then in Section~\ref{sec:stochastic}, we provide tighter bounds for the stochastic case.

\subsection{Adversarial setting}
\label{sec:adversarial}
We say that a prediction algorithm \emph{enjoys a regret or expected regret bound} $f:[0,\infty) \to \Reals$ under the given assumptions in the non-delayed setting if (i) $f$ is nondecreasing, concave, $f(0)=0$; and (ii) 
$\sup_{b_1,\ldots,b_n \in \OutcomeSet} R_n \le f(n)$ or, respectively, $\sup_{b_1,\ldots,b_n \in \OutcomeSet} \ExpVal{R_n} \le f(n)$ for all $n$. 
The algorithm of \citet{Weinberger:2002} for the adversarial full information setting
subsamples the reward sequence by the constant delay $\tau_{const}+1$, and runs a base algorithm \Base on each of the $\tau_{const}+1$ subsampled sequences. 
~\citet{Weinberger:2002} showed that if \Base enjoys a regret bound $f$ then their algorithm in the fixed delay case enjoys a regret bound $(\tau_{const}+1) f(n/(\tau_{const}+1))$. Furthermore, when \Base is minimax optimal in the non-delayed setting, the subsampling algorithm is also minimax optimal in the (full information) delayed setting, as can be seen by constructing a reward sequence that changes only in every $\tau_{const}+1$ times. Note that \citet{Weinberger:2002} do not require condition (i) of $f$. However, these conditions imply that $y f(x/y)$ is a concave function of $y$ for any fixed $x$ (a fact which will turn out to be useful in the analysis later), and are satisfied by all regret bounds we are aware of (e.g., for multi-armed bandits, contextual bandits, partial monitoring, etc.), which all have a regret upper bound of the form $\widetilde{O}(n^\alpha)$ for some $0 \le \alpha \le 1$, with, typically, $\alpha=1/2$ or  $2/3$.\footnote{$u_n=\widetilde{O}(v_n)$ means that there is a $\beta \ge 0$ such that $\lim_{n\to\infty} u_n/(v_n \log^\beta n)=0$.}. 

In this section we extend the algorithm of \citet{Weinberger:2002} to the case when the delays are not constant, and to the partial monitoring setting. The idea is that we run several instances of a non-delayed algorithm \Base as needed: an instance is ``free'' if it has received the feedback corresponding to its previous prediction -- before this we say that the instance is ``busy'', waiting for the feedback. When we need to make a prediction, we use one of existing instances that is free, and is hence ready to make another prediction. If no such instance exists, we create a new one to be used (a new instance is always ``free'', as it is not waiting for the feedback of a previous prediction). The resulting algorithm, which we call Black-Box Online Learning under Delayed feedback (BOLD) is shown below (note that when the delays are constant, BOLD reduces to the algorithm of~\citet{Weinberger:2002}): 

\begin{algorithm}[h]
\begin{algorithmic}
\FOR { \textbf{each} time instant $t = 1,2,\dots,n$ }
	\STATE {\bf Prediction:}
	\STATE \ \ Pick a free instance of \Base (independently of past predictions), or create a new instance if all existing  instances are busy. Feed the instance picked with $x_t$ and use its prediction.
	\STATE {\bf Update:}
	\FOR {\textbf{each} $ (s,h_s) \in H_t$}
		\STATE Update the instance used at time instant $s$ with the feedback $h_s$.
	\ENDFOR
\ENDFOR
\end{algorithmic}
\caption{Black-box Online Learning under Delayed feedback (\textbf{BOLD})}
\label{alg:BOLD}
\end{algorithm}

Clearly, the performance of BOLD depends on how many instances of \Base we need to create, and
how many times each instance is used. Let $M_t$ denote the number of \Base instances created by BOLD up to and including time $t$. That is, $M_1=1$, and we create a new instance
at the beginning of any time instant when all instances are waiting for their feedback. Let $G_t=\sum_{s=1}^{t-1} \Event{s+\tau_s \ge t}$ be the total number of outstanding (missing) feedbacks when the forecaster is making a prediction at time instant $t$. Then we have $G_t$ algorithms waiting for their feedback, and so $M_t \ge G_t+1$. Since we only introduce new instances when it is necessary (and each time instant at most one new instance is created), it is easy to see that 
\begin{equation}
\label{eq:Mt-Gt+1}
M_t=G^*_t+1
\end{equation}
for any $t$, where $G^*_t=\max_{1 \le s \le t} G_t$.

We can use the result above to transfer the regret guarantee of the non-delayed base algorithm \Base to a guarantee on the regret of BOLD. 

\begin{theorem}\label{thm:BOLD-Expected-Bound}
Suppose that the non-delayed algorithm \Base used in BOLD enjoys an (expected) regret bound $f_\Base$. 
Assume, furthermore, that the delays $\tau_t$ are independent of the forecaster's prediction $a_t$.
Then the expected regret of BOLD after $n$ time steps satisfies
\begin{align*} 
\ExpVal{R_n} &\le   \ExpVal{(G^*_n+1) f_\Base\left(\dfrac{n}{G^*_n+1}\right)} \nonumber \\
&\le (\ExpVal{G^*_n}+1) f_\Base\left(\dfrac{n}{\ExpVal{G^*_n}+1}\right).
\end{align*}
\end{theorem}
\begin{proof}
As the second inequality follows from the concavity of $y\mapsto yf_\Base(x/y)$ ($x,y> 0$),
 it remains to prove the first one.

For any $1 \le j \le M_n$, let $L_j$ denote the list of time instants in which BOLD has used the prediction chosen by instance $j$, and let
$n_j=|L_j|$ be the number of time instants this happens. Furthermore, let $R_{n_j}^j$ denote the regret incurred during the time instants $t$ with $t\in L_j$:
\[
R_{n_j}^j = \sup_{a\in \mathcal{F}} \sum_{t \in L_j}^{} r_t(x_t,a(x_t)) - \sum_{t \in L_j}^{} r_t(x_t,a_t),
\]
where $a_t$ is the prediction made by BOLD (and instance $j$) at time instant $t$.
By construction, instance $j$ does not experience any delays. Hence, $R_{n_j}^j$ is its regret in a non-delayed online learning problem. 
\footnote{Note that $L_j$ is a function of the delay sequence and is not a function of the predictions $(a_t)_{t\ge 1}$.
Hence, the reward sequence that instance $j$ is evaluated on is chosen obliviously whenever the adversary of BOLD is oblivious.} Then,
\begin{eqnarray*} 
 R_{n} & = & \sup_{a\in \mathcal{F}} \sum_{t=1}^{n} r_t(x_t,a(x_t)) - \sum_{t=1}^n r_t(x_t,a_t)\nonumber\\
& =& \sup_{a\in \mathcal{F}}\sum_{j=1}^{M_n} \sum_{t \in L_j}^{} r_t(x_t,a(x_t)) - \sum_{j=1}^{M_n} \sum_{t \in L_j}^{} r_t(x_t,a_t)\nonumber\\
& \le& \sum_{j=1}^{M_n} \left( \sup_{a\in \mathcal{F}} \sum_{t \in L_j}^{} r_t(x_t,a(x_t)) - \sum_{t \in L_j}^{} r_t(x_t,a_t)\right) \nonumber\\
& = &\sum_{j=1}^{M_n} R_{n_j}^j.
\end{eqnarray*}
Now, using the fact that $f_\Base$ is an (expected) regret bound, we obtain
\begin{align*}
\lefteqn{\ExpVal{R_n|\tau_1,\ldots,\tau_n} 
	 \le 
	\sum_{j=1}^{M_n} \ExpVal{R_{n_j}^j |\tau_1,\ldots,\tau_n}} \nonumber\\
	& \qquad\le \sum_{j=1}^{M_n} f_\Base(n_j) 
	= M_n \sum_{j=1}^{M_n} \dfrac{1}{M_n} f_\Base(n_j)\nonumber\\
& \qquad\le M_n f_\Base\left(\sum_{j=1}^{M_n} \dfrac{1}{M_n} n_j \right) = M_n f_\Base\left(\dfrac{n}{M_n}\right),
\end{align*}
where the first inequality follows since $M_n$ is a deterministic function of the delays, while
 the last inequality follows from Jensen's inequality and the concavity of $f_\Base$. Substituting $M_n$ from \eqref{eq:Mt-Gt+1} and taking the expectation concludes the proof. 
\end{proof}

Now, we need to bound $G^*_n$ to make the theorem meaningful. 
When all delays are the same constants, for $n > \tau_{const}$ we get $G^*_n=\tau_t=\tau_{const}$, and we get back the regret bound
\[
\ExpVal{R_n} \le (\tau_{const}+1) f_\Base\left( \dfrac{n}{\tau_{const}+1}\right)
\]
of \citet{Weinberger:2002}, thus generalizing their result to partial monitoring. We do not know whether this bound is tight even when \Base is minimax optimal, as the argument of \citet{Weinberger:2002} for the lower bound does not work in the partial information setting (the forecaster can gain extra information in each block with the same reward functions).

Assuming the delays are i.i.d., we can give an interesting bound on $G^*_n$. The result is based on the fact that although 
$G_t$ can be as large as $t$,  both its expectation and variance are upper bounded by $\ExpVal{\tau_1}$.

\begin{lemma}
\label{lem:exptau}
Assume $\tau_1,\ldots,\tau_n$ is a sequence of i.i.d. random variables with finite expected value, and let $B(n,t)=t+2\log n + \sqrt{4 t \log n}$. Then
\[
\ExpVal{G^*_n} \le B(n,\ExpVal{\tau_1}) + 1.
\]
\end{lemma}
\begin{proof}
First consider the expectation and the variance of $G_t$. For any $t$, 
\begin{align*}
\ExpVal{G_t}&=\ExpVal{\sum_{s=1}^{t-1} \Event{s+\tau_s \ge t}} 
=\sum_{s=1}^{t-1} \Prob{s+\tau_s \ge t} \\
&=\sum_{s=0}^{t-2} \Prob{\tau_1 > s} \le \ExpVal{\tau_1},
\end{align*}
and, similarly
\begin{align*}
\Variance{G_t}&=\sum_{s=1}^{t-1} \Variance{\Event{s+\tau_s \ge t}} \le\sum_{s=1}^{t-1} \Prob{s+\tau_s \ge t},
\end{align*}
so $\Variance{G_t} \le \ExpVal{\tau_1}$ in the same way as above.
By Bernstein's inequality \citep[Corollary~A.3]{Cesa-Bianchi:2006}, for any $0<\delta<1$ and any $t$ we have, with probability at least $1-\delta$,
\[
G_t -\ExpVal{G_t} \le \log \tfrac1\delta + \sqrt{2\Variance{G_t} \log\tfrac1\delta}.
\]
Applying the union bound for $\delta=1/n^2$, and our previous bounds on the variance and expectation of $G_t$, we obtain that with probability at least $1-1/n$,
\[
\max_{1\le t \le n} G_t \le \ExpVal{\tau_1} + 2\log n + \sqrt{4\ExpVal{\tau_1} \log n}.
\]
Taking into account that $\max_{1\le t \le n} G_t \le n$, we get the statement of the lemma.
\end{proof}

\begin{corollary}
Under the conditions of Theorem~\ref{thm:BOLD-Expected-Bound}, if the sequence of delays is i.i.d, then
\[
\ExpVal{R_n} \le (B(n,\ExpVal{\tau_1})+2) f_\Base\left(\dfrac{n}{B(n,\ExpVal{\tau_1})+2}\right).
\]
\end{corollary}
Note that although the delays can be arbitrarily large, whenever the expected value is finite, the bound only increases by a $\log n$ factor.

\subsection{Finite stochastic setting}\label{sec:stochastic}
In this section, 
	we consider the case when the prediction set $\ActionSet$ of the forecaster is finite; 
	without loss of generality we assume $\ActionSet=\{1,2,\dots,K\}$. 
We also assume that there is no side information 
	(that is, $x_t$ is a constant for all $t$, and, hence, will be omitted; 
	the results can be extended easily to the case of a finite side information set, 
	where we can repeat the procedures described below for each value of the side information separately).
The main assumption in this section is that the outcomes $(b_t)_{t\ge 1}$ form an i.i.d. sequence, 
	which is also independent of the predictions of the forecaster. 
When $\OutcomeSet$ is finite, this leads to the standard i.i.d. partial monitoring (IPM) setting, 
	while the conventional multi-armed bandit (MAB) setting is recovered 
	when the feedback is the reward of the last prediction, that is, $h_t=r_t(a_t,b_t)$.
As in the previous section, we will assume that the feedback delays are independent of the outcomes of the environment. 
The main result of this section shows that under these assumptions, 
	the penalty in the regret grows in an additive fashion due to the delays, 
	as opposed to the multiplicative penalty that we have seen in the adversarial case. 

By the independence assumption on the outcomes, the sequences of potential rewards $r_t(i)\eqdef r(i,b_t)$ and feedbacks $h_t(i) \eqdef h(i,b_t)$ are i.i.d., respectively, for the same prediction $i\in\ActionSet$. In this setting we also assume that the feedback and reward sequences of different predictions are independent of each other. Let $\mu_i=\ExpVal{r_t(i)}$ denote the expected reward of predicting $i$, $\mu^*=\max_{i\in\ActionSet} \mu_i$ the optimal reward and $i^*$ with $\mu_{i^*}=\mu^*$ the optimal prediction. 
Moreover, let $T_i(n) = \sum_{t=1}^n \Event{a_t=i}$ denote the number of times $i$ is predicted by the end of time instant $n$. Then, defining the ``gaps'' $\Delta_i = \mu^* - \mu_i$ for all $i \in \ActionSet$, the expected regret of the forecaster becomes
\begin{equation}\label{eq:Stochastic-Exp-Regret}
\ExpVal{R_n} = \sum_{t=1}^n \mu^* - \mu_{a_t} = \sum_{i=1}^K \Delta_i \ExpVal{T_i(n)}.
\end{equation}

Similarly to the adversarial setting, we build on a base algorithm \Base for the non-delayed case. The advantage in the IPM setting (and that we consider expected regret) is that here \Base can consider a permuted order of rewards and feedbacks, and so we do not have to wait for the actual feedback; it is enough to receive a feedback for the same prediction.
This is the idea at the core of our algorithm, Queued Partial Monitoring with Delayed Feedback (\StochBlackBoxAlg):
\begin{algorithm}[ht]
\begin{algorithmic}
\STATE Create an empty FIFO buffer $Q[i]$ for each $i \in \ActionSet$.
\STATE Let $I$ be the first prediction of \Base.
\FORALL {time instant $t =1,2,\dots,n$ }
	\STATE {\bf Predict:}
	\WHILE {$Q[I]$ is  not empty}
		\STATE Update \Base\ with a feedback from Q[I].
		\STATE Let $I$ be the next prediction of \Base.
	\ENDWHILE
	\STATE There are no buffered feedbacks for $I$, so predict $a_t=I$ at time instant $t$ to get a feedback.
	\STATE {\bf Update:}
	\FOR {\textbf{each} $(s,h_s) \in H_t$}
		\STATE Add the feedback $h_s$ to the buffer $Q[a_s]$.
	\ENDFOR
\ENDFOR
\end{algorithmic}
\caption{Queued Partial Monitoring with Delays (\StochBlackBoxAlg)}
\label{alg:Stochastic-Blackbox}
\end{algorithm}

Here we have a \Base partial monitoring algorithm for the non-delayed case, which is run inside the algorithm. The feedback information coming from the environment is stored in
separate queues for each prediction value. The outer algorithm constantly queries \Base: while feedbacks for the predictions made are available in the queues, only the inner algorithm \Base runs
(that is, this happens within a single time instant in the real prediction problem). When no feedback is available, the outer algorithm keeps sending the same prediction to the real environment until a feedback for that prediction arrives. In this way \Base is run in a simulated non-delayed environment. 
The next lemma implies that the inner algorithm \Base actually runs in a non-delayed version of the problem, as it experiences the same distributions:
\begin{lemma}\label{thm:DelayReordering}
Consider a delayed stochastic IPM problem as defined above. For any prediction $i$, for any $s \in \Naturals$ let $h'_{i,s}$ denote the $\thtimes{s}$ feedback \StochBlackBoxAlg receives for predicting $i$. Then the sequence $(h'_{i,s})_{s \in \Naturals}$ is an i.i.d. sequence with the same distribution as the sequence of feedbacks $(h_{t,i})_{t \in \Naturals}$ for prediction $i$.
\end{lemma}

To relate the non-delayed performance of \Base and the regret of \StochBlackBoxAlg, we need a few definitions.
For any $t$, let $S_i(t)$ denote the number of feedbacks for prediction $i$ that are received by the end of time instant $t$. 
Then the number of missing feedbacks for $i$ 
	when making a prediction at time instant $t$ is $G_{i,t} = T_i(t-1) - S_i(t-1)$. 
Let $G^*_{i,n} = \max_{1 \le t \le n} G_{i,t}$.
Furthermore, for each $i \in \ActionSet$, let $T'_i(t')$ be the number of times algorithm \Base has predicted $i$ while being queried $t'$ times.
Let $n'$ denote the number of steps the inner algorithm \Base makes in $n$ steps of the real IPM problem. Next we relate $n$ and $n'$, as well as the number of times
\StochBlackBoxAlg and \Base (in its simulated environment) make a specific prediction.
\begin{lemma}\label{lem:Outstanding-Rewards-UpperBound}
Suppose \StochBlackBoxAlg is run for $n \ge 1$ time instants, and has queried \Base  $n'$ times. Then
$n' \le n$ and
\begin{equation}\label{eq:Outstanding-Rewards-UpperBound}
0 \le {T_i(n) - T'_i(n')} \le G^*_{i,n}.
\end{equation}
\end{lemma}
\vspace*{-0.3cm}
\begin{proof}
Since \Base can take at most one step for each feedback that arrives, and  \StochBlackBoxAlg has to make at least one step for each arriving feedback, $n' \le n$.

Now, fix a prediction $i\in \ActionSet$.
If \Base, and hence,  \StochBlackBoxAlg, has not predicted $i$ by time instant $n$, \eqref{eq:Outstanding-Rewards-UpperBound} trivially holds. Otherwise,
let $t_{n,i}$ denote the last time instant (up to time $n$) when \StochBlackBoxAlg predicts $i$. Then $T_i(n) = T_i(t_{n,i}) = T_i(t_{n,i}-1)+1$. Suppose \Base has been queried $n'' \le n$ times by time instant $t_{n,i}$ (inclusive). At this time instant, the buffer $Q[i]$ must be empty and \Base must be predicting $i$, otherwise \StochBlackBoxAlg would not predict $i$ in the real environment. 
This means that all the $S_i(t_{n,i}-1)$ feedbacks that have arrived before this time instant have been fed to the base algorithm, which has also made an extra step, that is, 
$T'_i(n') \ge T'_i(n'') = S_i(t_{n,i}-1)+1$. Therefore,
\begin{align*}
T_i(n) - T'_i(n') & \le T_i(t_{n,i}-1) + 1 - (S_i(t_{n,i}-1) +1)\\
& \le G_{i,t_{n,i}} \le G^*_{i,n}. \qedhere
\end{align*}
\end{proof}

We can now give an upper bound on the expected regret of Algorithm~\ref{alg:Stochastic-Blackbox}.

\begin{theorem}\label{thm:Q-Alg-Expected-Regret}
Suppose the non-delayed \Base algorithm is used in \StochBlackBoxAlg in a delayed stochastic IPM environment. Then the expected regret of \StochBlackBoxAlg is upper-bounded by
\begin{equation}\label{eq:Q-Alg-Expected-Regret}
\ExpVal{R_n} \le \ExpVal{ R^\textsc{Base}_n } + \sum_{i=1}^{K} \Delta_i \ExpVal{ G^*_{i,n} },
\end{equation}
where $\ExpVal{R^\textsc{Base}_n}$ is the expected regret of \Base when run in the same environment without delays.
\end{theorem}
When the delay $\tau_t$ is bounded by $\tau_{max}$ for all $t$, we also have $G^*_{i,n} \le \tau_{max}$, and $\ExpVal{R_n} \le \ExpVal{ R^\textsc{Base}_n } + O(\tau_{max})$.
When the sequence of delays for each prediction is i.i.d. with a finite expected value but unbounded support, we can use Lemma~\ref{lem:exptau} to bound $G^*_{i,n}$, and obtain a bound $\ExpVal{ R^\textsc{Base}_n } + O(\ExpVal{\tau_1}+\sqrt{\ExpVal{\tau_1}\log n} + \log n)$. 
\mbox{}\vspace*{-0.2cm}\mbox{}
\begin{proof}
Assume that  \StochBlackBoxAlg is run longer so that  \Base is queried for $n$ times (i.e., it is queried $n-n'$ more times). Then, since $n' \le n$, the number of times $i$ is predicted by the base algorithm, namely $T'_i(n)$, can only increase, that is, $T'_i(n') \le T'_i(n)$. Combining this with the expectation
of \eqref{eq:Outstanding-Rewards-UpperBound} gives
\begin{equation*}
\ExpVal{T_i(n)} \le \ExpVal{ T'_i(n) } + \ExpVal{ G^*_{i,n} },
\end{equation*}
which in turn gives,
\begin{align}\label{eq:WaldRegret-BB}
\sum_{i=1}^{K} & \Delta_i \ExpVal{T_i(n)}
 \le \sum_{i=1}^{K} \Delta_i \ExpVal{ T'_i(n)} + \sum_{i=1}^{K} \Delta_i \ExpVal{ G^*_{i,n}}.
\end{align}
As shown in Lemma~\ref{thm:DelayReordering}, the reordered rewards and feedbacks $h'_{i,1}, h'_{i,2},  \dots, h'_{i,T'_i(n')}, \dots h'_{i,T_i(n)}$ are i.i.d. with the same distribution as the original feedback sequence $(h_{t,i})_{t \in \Naturals}$. The base algorithm \Base has worked on the first $T'_i(n)$ of these feedbacks for each $i$ (in its extended run), and has therefore operated for $n$ steps in a simulated environment with the same reward and feedback distributions, but without delay. Hence, the first summation in the right hand side of~\eqref{eq:WaldRegret-BB} is in fact $\ExpVal{R^\text{Base}_n}$, the expected regret of the base algorithm in a non-delayed environment. This concludes the proof.
\end{proof}

\section{UCB for the Multi-Armed Bandit Problem with Delayed Feedback}
\label{sec:UCB}
While the algorithms in the previous section provide an easy way to convert algorithms devised for the non-delayed case to ones that can handle delays in the feedback, improvements can be achieved if one makes modifications inside the existing non-delayed algorithms while retaining their theoretical guarantees. This can be viewed as a "white-box" approach to extending online learning algorithms to the delayed setting, and enables us to escape the high memory requirements of black-box algorithms that arises for both of our methods in the previous section when the delays are large. We consider the stochastic multi-armed bandit problem, and extend the UCB family of algorithms \citep{Auer:2002,Garivier:2011} to the delayed setting. The modification proposed is quite natural, and the common characteristics of UCB-type algorithms enable a unified way of extending their performance guarantees to the delayed setting (up to an additive penalty due to delays).

Recall that in the stochastic MAB setting, which is a special case of the stochastic IPM problem of Section \ref{sec:stochastic}, the feedback at time instant $t$ is $h_t=r(a_t, b_t)$, and there is a distribution $\nu_i$ from which the rewards of each prediction $i$ are drawn in an i.i.d. manner. Here we assume that the rewards of different predictions are independent of each other. We use the same notation as in Section \ref{sec:stochastic}.

Several algorithms devised for the non-delayed stochastic MAB problem are based on upper confidence bounds (UCBs), which are optimistic estimates of the expected reward of different predictions. Different UCB-type algorithms use different upper confidence bounds, and choose, at each time instant, a prediction with the largest UCB. 
Let $B_{i,s,t}$ denote the UCB for prediction $i$ at time instant $t$, where $s$ is the number of reward samples used in computing the estimate. In a non-delayed setting, the prediction of a UCB-type algorithm at time instant $t$ is given by
$a_t = \text{argmax}_{i \in \ActionSet}\ B_{i,T_i(t-1),t}.$
In the presence of delays, one can simply use the same upper confidence bounds only with the rewards that are observed, and predict
\begin{equation} \label{eq:UCB-Delayed-General-Rule}
a_t = \text{argmax}_{i \in \ActionSet}\ B_{i,S_i(t-1),t}
\end{equation}
at time instant $t$ (recall that $S_i(t-1)$ is the number of rewards that can be observed for prediction $i$ before time instant $t$). Note that if the delays are zero, this algorithm reduces to the corresponding non-delayed version of the algorithm. 

The algorithms defined by \eqref{eq:UCB-Delayed-General-Rule} can easily be shown to enjoy the same regret guarantees compared to their non-delayed versions, up to an additive penalty depending on the delays. This is because the analyses of the regrets of UCB algorithms follow the same pattern of upper bounding the number of trials of a suboptimal prediction using  concentration inequalities suitable for the specific form of UCBs they use. 

As an example, the UCB1 algorithm \citep{Auer:2002} uses UCBs of the form
$B_{i,s,t} = \hat{\mu}_{i,s} + \sqrt{{2 \log(t)}/{s}}$,
where $\hat{\mu}_{i,s} = \tfrac{1}{s} \sum_{t=1}^{s} h'_{i,t}$ is the average of the first $s$ observed rewards. Using this UCB in our decision rule
\eqref{eq:UCB-Delayed-General-Rule}, we can bound the regret of the resulting algorithm (called Delayed-UCB1) in the delayed setting:
\begin{theorem}\label{thm:UCB-Delayed-Expected-Bound}
For any $n \ge 1$, the expected regret of the Delayed-UCB1 algorithm is bounded by
\vspace{-0.1cm}
\begin{align*}
\ExpVal{R_n} \le  
& \sum_{i: \Delta_i > 0}\left[ \dfrac{8 \log{n}}{\Delta_i} + 3.5\Delta_i \right] 
+
 \sum_{i=1}^{K} \Delta_i \ExpVal{G_{i,n}^*}.
\end{align*}
\mbox{}\vspace*{-0.2cm}\mbox{}
\end{theorem}
Note that the last term in the bound is the additive penalty, and, under different assumptions, it can be bounded in the same way as after Theorem~\ref{thm:Q-Alg-Expected-Regret}.
The proof of this theorem, as well as a similar regret bound for the delayed version of the KL-UCB algorithm~\citep{Garivier:2011} can be found in Appendix \ref{apx:GeneralUCBScheme}.

\section{Conclusion and future work}
We analyzed the effect of feedback delays in online learning problems. We examined the partial monitoring case (which also covers the full information and the bandit settings), and provided general algorithms that transform forecasters devised for the non-delayed case into ones that handle delayed feedback. It turns out that the price of delay is a multiplicative increase in the regret in adversarial problems, and only an additive increase in stochastic problems. While we believe that these findings are qualitatively correct, we do not have lower bounds to prove this (matching lower bounds are available for the full information case only).

It also turns out that the most important quantity that determines the performance of our algorithms is $G^*_n$, the maximum number of missing rewards. It is interesting to note that $G^*_n$ is the maximum number of servers used in a multi-server queuing system with infinitely many servers and deterministic arrival times. It is also the maximum deviation of a certain type of Markov chain. While we have not found any immediately applicable results in these fields, we think that applying techniques from these areas could lead to an improved understanding of $G^*_n$, and hence an improved analysis of online learning under delayed feedback.

 \section{Acknowledgements}
This work was supported by the Alberta Innovates Technology Futures and NSERC.

\bibliography{database}
\bibliographystyle{icml2013}
\onecolumn

\appendix

\section{Proof of Lemma~\ref{thm:DelayReordering} }
In this appendix we prove Lemma~\ref{thm:DelayReordering} that was used in the i.i.d. partial monitoring setting (Section~\ref{sec:stochastic}).
To that end, we will first need two other lemmas. The first lemma shows that the i.i.d. property of a sequence of random variables is preserved under an independent random reordering of that sequence. 

\begin{lemma}\label{lem:i.i.d.Reordering}
Let $(X_t)_{t \in \Naturals}$, be a sequence of independent, identically distributed random variables. If we reorder this sequence according to an independent random permutation, then the resulting sequence is i.i.d. with the same distribution as $(X_t)_{t \in \Naturals}$.
\end{lemma}

\begin{proof}
Let the reordered sequence be denoted by $(Z_t)_{t \in \Naturals}$. It is sufficient to show that for all $n \in \Naturals$, for all $y_1, y_2, \dots, y_n$, we have
\begin{align*}
\Prob{ Z_1 \le y_1, Z_2 \le y_2, \dots, Z_n \le y_n} = 
\Prob{X_1 \le y_1, X_2 \le y_2, \dots, X_n \le y_n}.
\end{align*}
Since $(X_t)_{t \in \Naturals}$ is i.i.d., for any fixed permutation the equation above holds as both sides are equal to $\Pi_{t=1}^{n} \Prob{X_1 \le y_t}$. Since the permutations are independent of the sequence $(X_t)_{t \in \Naturals}$, using the law of total probability this extends to the general case as well.
\end{proof}

We also need the following result~\citep[Page 145, Chapter III, Theorem 5.2]{Doob:1953}.

\begin{lemma}\label{lem:i.i.d.SequentialChoice}
Let $(X_{t})_{t \in \Naturals}$ be a sequence of i.i.d. random variables, and $(X'_{t})_{t \in \Naturals}$ be a subsequence of it such that the decision whether to include $X_{t}$ in the subsequence is independent of future values in the sequence, i.e., of $X_{s}$ for $s \ge t$. Then the sequence $(X'_{t})_{t \in \Naturals}$ is an i.i.d. sequence with the same distribution as $(X_{t})_{t \in \Naturals}$.
\end{lemma}

We can now proceed to the proof of Lemma~\ref{thm:DelayReordering}.
\begin{proof}[Proof of Lemma~\ref{thm:DelayReordering}]
Let $(Z_{i,t})_{t \in \Naturals}$ be the sequence resulting from sorting the variables $h_{i,t}$ by their \emph{possible} observation times $t+\tau_{i,t}$ (that is, $Z_{i,1}$ is the earliest feedback that can be observed if $i$ is predicted at the appropriate time, and so on). Since delays are independent of the outcomes, they define an independent reordering on the sequence of feedbacks. Hence, by Lemma~\ref{lem:i.i.d.Reordering}, $(Z_{i,t})_{t \in \Naturals}$ is an i.i.d. sequence with the same distribution as $(h_{i,t})_{t \in \Naturals}$. Note that $(h'_{i,s})_{s \in \Naturals}$, the sequence of feedbacks (sorted by their observation times) that the agent observes for predicting $i$, is a subsequence of $(Z_{i,t})_{t \in \Naturals}$ where the decision whether to include each $Z_{i,t}$ in the subsequence cannot depend on future possible observations $Z_{i,t'}, t' \ge t$. Also, the feedbacks of other predictions that are used in this decision were assumed to be independent of $(Z_{i,t})_{t \in \Naturals}$. Hence, by Lemma \ref{lem:i.i.d.SequentialChoice}, $(h'_{i,s})_{s \in \Naturals}$ is an i.i.d. sequence with the same distribution as $(Z_{i,t})_{t \in \Naturals}$, which in turn has the same distribution as $(h_{i,t})_{t \in \Naturals}$.
\end{proof}

\section{UCB for the Multi-Armed Bandit Problem with Delayed Feedback}\label{apx:GeneralUCBScheme}
This appendix details the framework we described in Section~\ref{sec:UCB} for analyzing UCB-type algorithms in the delayed settings, and provides the missing proofs.
The regret of a UCB algorithm is usually analyzed by upper bounding the (expected) number of times a suboptimal prediction is made, and then using Equation~\eqref{eq:Stochastic-Exp-Regret} to get an expected regret bound. Consider a UCB algorithm with upper confidence bounds $B_{i,s,t}$, and fix a suboptimal prediction $i$. The typical analysis (e.g., by \citet{Auer:2002}) considers the case when this prediction is made for at least $\ell > 1$ times (for a large enough $\ell$), and uses concentration inequalities suitable for the specific form of the upper-confidence bound to show that it is unlikely to make this suboptimal prediction more than $\ell$ times because observing $\ell$ samples from its reward distribution suffices to distinguish it from the optimal prediction with high confidence. This value $\ell$ thus gives an upper bound on the expected number of times $i$ is predicted. Examples of such concentration inequalities include Hoeffding's inequality~\cite{Hoeffding:1963} and Theorem 10 of \citet{Garivier:2011}, which are used for the UCB1 and KL-UCB algorithms, respectively.

More precisely, the general analysis of UCB-type algorithms in the non-delayed setting works as follows: for $\ell>1$, we have  $T_i(n) \le \ell + \sum_{t=1}^{n} \Event{a_t=i,\ T_i(t)>\ell}$,
where the sum on the right hand side captures how much larger than $\ell$ the value of $T_i(n)$ is (recall that $T_i(t)$ is the number of times $i$ is predicted up to and including time $t$). Whenever $i$ is predicted, its UCB, $B_{i,T_i(t-1),t}$, must have been greater than that of an optimal prediction, $B_{i^*,\ T_{i^*}(t-1),t}$, which implies
\begin{align}\label{eq:General-UCB-NonDelayed-Decompose}
T_i(n) \le \ell + \sum_{t=1}^{n} \SimpleEvent \{&
 B_{i,T_i(t-1),t} \ge B_{i^*, T_{i^*}(t-1),t}, a_t=i, \ T_i(t-1) \ge \ell \}.
\end{align}
The expected value of the summation on the right-hand-side is then bounded using concentration inequalities as mentioned above.

In the delayed-feedback setting, if we use upper confidence bounds $B_{i,S_i(t-1),t}$ instead (where $S_i(t)$ was defined to be the number of rewards observed up to and including time instant $t$), in the same way as above we can write
\begin{align*}
T_i(n) \le \ell + \sum_{t=1}^{n} \SimpleEvent \{ & B_{i,S_i(t-1),t} \ge B_{i^*, S_{i^*}(t-1),t}, \ a_t=i, \ T_i(t-1) \ge \ell \}.
\end{align*}
Since $T_i(t-1) = G_{i,t} + S_i(t-1)$, with $\ell' = \ell - G^*_{i,n}$ we get
\begin{align}
T_i(n) \le \ell' + G^*_{i,n} + \sum_{t=1}^{n} &\ \SimpleEvent \{ B_{i,S_i(t-1),t} \ge B_{i^*, S_{i^*}(t-1),t}, a_t=i, \ S_i(t-1) \ge \ell' \}. \label{eq:General-UCB-Modification}
\end{align}
Now the same concentration inequalities used to bound \eqref{eq:General-UCB-NonDelayed-Decompose} in the analysis of the non-delayed setting can be used to upper bound the expected value of the sum in \eqref{eq:General-UCB-Modification}. Putting this into \eqref{eq:Stochastic-Exp-Regret}, we see that one can reuse the same upper confidence bound in the delayed setting (with only the observed rewards) and get a performance similar to the non-delayed setting, with only an additive penalty that depends on the delays. The following two sections demonstrate the use of this method on two UCB-type algorithms.

\subsection{UCB1 under delayed feedback: Proof of Theorem~\ref{thm:UCB-Delayed-Expected-Bound}}
Below comes the proof of Theorem~\ref{thm:UCB-Delayed-Expected-Bound} for the Delayed-UCB1 algorithm (Section \ref{sec:UCB}).
\begin{proof}[Proof of Theorem~\ref{thm:UCB-Delayed-Expected-Bound}]
Following the outline of the previous section, we can bound the summation in \eqref{eq:General-UCB-Modification} using the same analysis as in the original UCB1 paper \citep{Auer:2002}. In particular, for any prediction $i$ we can write
\begin{align}
& \sum_{t=1}^{n}\ \Event{B_{i,S_i(t-1),t} \ge B_{i^*, S_{i^*}(t-1),t},\ S_i(t-1) \ge \ell'}   \nonumber\\
 & \le \sum_{t=1}^{n}\ \Event{B_{i^*, S_{i^*}(t-1),t} \le \mu_{i^*},\ S_i(t-1) \ge \ell'}   + \sum_{t=1}^{n}\ \Event{B_{i,S_i(t-1),t} \ge \mu_{i^*} ,\ S_i(t-1) \ge \ell'}. \label{eq:UCB1-Mid1}
 \end{align}
The event in the second summation implies that either $\mu_i + 2 \sqrt{\dfrac{2 \log(t)}{S_i(t-1)}} > \mu_{i^*}$ or  $ \hat{\mu}_{i,S_i(t-1)} - \sqrt{\dfrac{2 \log(t)}{S_i(t-1)}} \ge \mu_i $ (otherwise we will have $B_{i,S_i(t-1),t} < \mu_{i^*}$). Hence,
{\allowdisplaybreaks 
\begin{align}
\eqref{eq:UCB1-Mid1} \le \ \ \sum_{t=1}^{n}\ \SimpleEvent \Bigg\{ & \hat{\mu}_{i^*,S_{i^*}(t-1)} + \sqrt{\dfrac{2 \log(t)}{S_{i^*}(t-1)}} \le \mu_{i^*} \Bigg\} +\nonumber \\
\ \sum_{t=1}^{n}\ \SimpleEvent \Bigg\{ & \hat{\mu}_{i,S_i(t-1)} - \sqrt{\dfrac{2 \log(t)}{S_i(t-1)}} \ge \mu_i \Bigg\} + \nonumber \\
\ \sum_{t=1}^{n}\ \SimpleEvent \Bigg\{ & \mu_i + 2 \sqrt{\dfrac{2 \log(t)}{S_i(t-1)}} > \mu_{i^*} , \ S_i(t-1) \ge \ell' \Bigg\}.
\label{eq:UCB1-Mid2}
\end{align} }
Choosing $\ell' = \ceiling{\dfrac{8 \log(n)}{\Delta_i^2}}$ makes the events in the last summation above impossible, because
$ S_i(t-1) \ge \ell' \ge \dfrac{8 \log(n)}{\Delta_i^2} \text{ which implies } 2\sqrt{ \dfrac{2 \log(t)}{S_i(t-1) }} \le 2\sqrt{ \dfrac{2 \log(n)}{\ell'}} \le \Delta_i
$. Therefore, combining with \eqref{eq:General-UCB-Modification}, we can write
\begin{align*}
T_i(n) \ \le\ &\  \ceiling{\dfrac{8 \log(n)}{\Delta_i^2}} + G_{i,n}^* + \ \sum_{t=1}^{n}\ \sum_{s=1}^{t} \Bigg( \Event{\hat{\mu}_{i^*,s}\ +\ \sqrt{\dfrac{2 \log(t)}{s}} \le \mu_{i^*}}  + \Event{\hat{\mu}_{i,s} - \sqrt{\dfrac{2 \log(t)}{s}} \ge \mu_i } \Bigg). \\
\intertext{Taking expectation gives}
\ExpVal{T_i(n)} \ \le\ &\  \ceiling{\dfrac{8 \log(t)}{\Delta_i^2}} + \ExpVal{G_{i,n}^*} +\ \sum_{t=1}^{n}\ \sum_{s=1}^{t} \Bigg( \Prob{\hat{\mu}_{i^*,s}\ +\ \sqrt{\dfrac{2 \log(t)}{s}} \le \mu_{i^*}}  + \Prob{\hat{\mu}_{i,s} - \sqrt{\dfrac{2 \log(t)}{s}} \ge \mu_i } \Bigg) .
\end{align*}
As in the original analysis, Hoeffding's inequality \citep{Hoeffding:1963} can be used to bound each of the probabilities in the summation, to get
\begin{align*}
& \Prob{\hat{\mu}_{i^*,s}\ +\ \sqrt{\dfrac{2 \log(t)}{s}} \le \mu_{i^*}} \le e^{-4 \log(t)} = t^{-4},  \\
& \Prob{\hat{\mu}_{i,s} - \sqrt{\dfrac{2 \log(t)}{s}} \ge \mu_i } \le e^{-4 \log(t)} = t^{-4}.
\end{align*}
Therefore, we have
\begin{align*}
\ExpVal{T_i(n)} \ \le\ &\  \ceiling{\dfrac{8 \log(t)}{\Delta_i^2}} + \ExpVal{G_{i,n}^*}\  + \ \sum_{t=1}^{\infty}\ 2t^{-3} 
\le \dfrac{8 \log(n)}{\Delta_i^2} +1 + \ExpVal{G_{i,n}^*}\ + 2 \zeta(3),
\end{align*}
where $\zeta(3) < 1.21$ is the Riemann Zeta function.\footnote{For properties and theory of the Riemann Zeta function, see the book of~\citet{Titchmarsh:1987}.}
Combining with \eqref{eq:Stochastic-Exp-Regret} proves the theorem.
\end{proof}

\subsection{KL-UCB under delayed feedback}\label{sec:KL-UCB-Modified}
The KL-UCB algorithm was introduced by~\citet{Garivier:2011}. The upper confidence bound used by KL-UCB for predicting $i$ at time $t$ is $B_{i,T_i(t-1),t}$ , where $B_{i,s,t}$ is
\begin{equation*}\label{eq:KL-UCB-Rule}
\max{} \left\lbrace q \in [\hat{\mu}_{i,s}, 1] : s\kldiv(\hat{\mu}_{i,s},q) \le \log t+3 \log(\log t) \right\rbrace,
\end{equation*}
with $\kldiv(p,q)=p\log(\tfrac{p}{q})+(1-p)\log(\tfrac{1-p}{1-q})$ the KL-divergence of two Bernoulli random variables with parameters $p$ and $q$. 
In their Theorem 2, \citet{Garivier:2011} show that there exists a constant $C_1 \le 10$, as well as functions $0 \le C_2(\epsilon) = O(\epsilon^{-2})$ and $0 \le \beta(\epsilon)=O(\epsilon^2)$, such that for any $\epsilon > 0$, the expected regret of the KL-UCB algorithm (in the non-delayed setting) satisfies
\begin{align}
\ExpVal{R_n} \le \sum_{\Delta_i >0} \Delta_i \Bigg[  & \dfrac{\log(n)}{\kldiv(\mu_i,\mu_{i^*})} (1+\epsilon)  + C_1 \log(\log n) +  \dfrac{C_2(\epsilon)}{n^{\beta(\epsilon)}} \Bigg]. \label{eq:NonDelayed-KL-UCB-Regret}
\end{align}

Using this upper confidence bound with \eqref{eq:UCB-Delayed-General-Rule}, we arrive at the Delayed-KL-UCB algorithm. For this algorithm, we can prove the following regret bound using the general scheme described above together with the same techniques used by \citet{Garivier:2011}, again obtaining an additive penalty compared to the non-delayed setting.

\begin{theorem}\label{thm:KL-UCB}
For any $\epsilon > 0$, the expected regret of the Delayed-KL-UCB algorithm after $n$ time instants satisfies
\begin{align*}
\ExpVal {R_n} \le & \sum_{i: \Delta_i>0}^{} \Delta_i \left(\dfrac{\log(n)}{\kldiv(\mu_i,\mu_{i^*})} (1 + \epsilon) +
C_1 \log(\log(n)) \right) + \sum_{i=1}^{K} \Delta_i \left(\dfrac{C_2(\epsilon)}{n^{\beta(\epsilon)}} \ExpVal{G^*_{i,n}} +\ExpVal{G^*_{i,n}} + 1\right),
\end{align*}
where $C_1, C_2$, and $\beta$ are the same as in \eqref{eq:NonDelayed-KL-UCB-Regret}.
\end{theorem}

In this case, working out the proof and reusing the analysis is somewhat more complicated compared to UCB1. In particular, we will need an adaptation of Lemma 7 of \citet{Garivier:2011}, which is captured by the following lemma.

\begin{lemma}\label{lem:KL-7-Adapted}
Let $
\kldiv^+(x,y) = \kldiv(x,y) \Event{ x < y }
$. Then for any $n \ge 1$,
\begin{align*}
{ \sum_{t=1}^{n} \Event{ a_t=i, \mu_{i^*} \le B_{i^*,S_{i^*}(t-1),t}, \ S_i(t-1) \ge \ell' } \le}  \ G^*_{i,n} \sum_{s=\ell'}^{n} \Event{ s \kldiv^+(\hat{\mu}_{i,s}, \mu_{i^*}) < \log(n) + 3 \log(\log(n))}.
\end{align*}
\end{lemma}

\begin{proof}[Proof of Lemma \ref{lem:KL-7-Adapted}]
We start in the same way as the original proof. Note that $\kldiv^+(p,q)$ is non-decreasing in its second parameter, and that $a_t=i$ and $\mu_{i^*} \le B_{i^*,S_{i^*}(t-1),t}$ together imply $B_{i,S_i(t-1),t} \ge B_{i^*,S_{i^*}(t-1),t} \ge \mu_{i^*}$, which in turn gives \[S_i(t-1)\kldiv^+(\hat{\mu}_{i,S_i(t-1)}, \mu_{i^*}) \le S_i(t-1)\kldiv(\hat{\mu}_{i,S_i(t-1)}, B_{i,S_i(t-1),t}) \le {\log(t) + 3 \log(\log(t))}.\]
Therefore, we have
{\allowdisplaybreaks 
\begin{align*}
\lefteqn{ 
\sum_{t=1}^{n} \Event{ a_t=i, \mu_{i^*} \le B_{i^*,S_{i^*}(t-1),t},  \ t > S_i(t-1) \ge \ell' }} \\ 
&\le\sum_{t=\ell'}^{n}  \Event{ a_t=i, S_i(t-1) \kldiv^+(\hat{\mu}_{i,S_i(t-1)}, \mu_{i^*}) \le \log(t) + 3 \log(\log(t)),  \ S_i(t-1) \ge \ell' }\\
&\le  \sum_{t=\ell'}^{n} \Event{ a_t=i, S_i(t-1) \kldiv^+(\hat{\mu}_{i,S_i(t-1)}, \mu_{i^*}) \le \log(n) + 3 \log(\log(n)),  \ S_i(t-1) \ge \ell' }\\ 
&\le  \sum_{t=\ell'}^{n} \sum_{s=\ell'}^{t} \Event{ a_t=i, S_i(t-1)=s}  \times \Event{s\kldiv^+(\hat{\mu}_{i,s}, \mu_{i^*}) \le \log(n) + 3 \log(\log(n)) }\\
&= \sum_{s=\ell'}^{n} \sum_{t=s}^{n}  \Event{ a_t=i, S_i(t-1)=s}  \times  \Event{s\kldiv^+(\hat{\mu}_{i,s}, \mu_{i^*}) \le \log(n) + 3 \log(\log(n)) }\\
&=  \sum_{s=\ell'}^{n}  \Event{s\kldiv^+(\hat{\mu}_{i,s}, \mu_{i^*}) \le \log(n) + 3 \log(\log(n)) } \times \left(\sum_{t=s}^{n} \Event{ a_t=i, S_i(t-1)=s} \right).
\end{align*}
 }
But note that the second summation is bounded by $G^*_{i,n}$, because for each $s$, there cannot be more than $G^*_{i,n}$ time instants at which $i$ is predicted and $S_i(t)=s$ remained constant; otherwise for some $t' \in \{s, \dots, n\}$ we would have $T_i(t'-1) - S_i(t'-1) = G_{i,t'} > G^*_{i,n}$, which is not possible. Substituting this bound in the last expression proves the lemma.
\end{proof}

We also recall the following two results from the original paper. 
\begin{theorem}[Theorem 10 of \citet{Garivier:2011}]\label{thm:KL-UCB-Stat-Tool}
Let $( Y_t ), t \ge 1$ be a sequence of independent random variables bounded in $[0,1]$, with common expectation $\mu = \ExpVal{Y_t}$. Consider a sequence $(\epsilon_t), t\ge 1$ of Bernoulli variables such that for all $t>0$, $\epsilon_t$ is a random function of $Y_1,\dots,Y_{t-1}$ \footnote{That is, a function of $Y_1, \dots, Y_{t-1}$ together with possibly an extra, independent randomization.}, and is independent of $Y_s, s \ge t$. Let $\delta>0$ and for every $1 \le t \le n$, let
\[
S_t = \sum_{s=1}^{t} \epsilon_s \ \ \ \text{ and } \ \ \ \ \hat{\mu}_{t} = \dfrac{\sum_{s=1}^{t} \epsilon_s Y_s}{S_t},
\]
with $\hat{\mu}_t=0$ when $S_t=0$, and
\[
B_{n} = \max \left\lbrace q > \hat{\mu}_{n}: S_n \kldiv(\hat{\mu}_n,q) \le \delta \right\rbrace.
\]
Then
\[
\Prob{B_n < \mu} \le e \lceil \delta \log(n) \rceil e^{-\delta}.
\]
\end{theorem}

\begin{lemma}[Lemma 8 of \citet{Garivier:2011}]\label{lem:KL-8} 
For a suboptimal prediction $i$, for every $\epsilon>0$, let $K_n=\floor{\dfrac{1+\epsilon}{\kldiv^+(\mu_i,\mu_{i^*})} \Bigg( \log(n)+3\log(\log(n)) \Bigg)}$. Then there exist $C_2(\epsilon) > 0$ and $\beta(\epsilon) > 0 $ such that
\begin{equation*}
\sum_{s=K_n+1}^{\infty} \Prob{ \kldiv^+(\hat{\mu}_{i,s},\mu_{i^*} ) < \dfrac{\kldiv(\mu_i,\mu_{i^*})}{1+\epsilon} }  \le \dfrac{C_2(\epsilon)}{n^{\beta(\epsilon)}}.
\end{equation*}
\end{lemma}

Now, we are ready to prove Theorem \ref{thm:KL-UCB} by reusing the same techniques as in the original paper.
\begin{proof}[Proof of Theorem \ref{thm:KL-UCB}]
For a suboptimal prediction $i$, bounding the terms in \eqref{eq:General-UCB-Modification} gives
\begin{align}
 \lefteqn{\sum_{t=1}^{n} \Event{ a_t = i, B_{i,S_i(t-1),t} \ge B_{i^*,S_{i^*}(t-1),t},\ S_i(t-1) \ge \ell'} }\nonumber\\
& \le \sum_{t=1}^{n} \Event{ B_{i^*,S_{i^*}(t-1),t} < \mu_{i^*} }  + \sum_{t=1}^{n} \Event{ a_t=i,\ \mu_{i^*} \le B_{i^*,S_{i^*}(t-1),t},\ S_i(t-1) \ge \ell' }  \nonumber\\
& \le \sum_{t=1}^{n} \Event{ B_{i^*,S_{i^*}(t-1),t} < \mu_{i^*} } + G^*_{i,n} \sum_{s=\ell'}^{n} \Event{ s \kldiv^+(\hat{\mu}_{i,s}, \mu_{i^*}) < \log(n) + 3 \log(\log(n))}, \label{eq:KL-Proof-Lemma7-Result}
\end{align}
where the last inequality follows from Lemma \ref{lem:KL-7-Adapted}. Let 
\begin{align}
K_n & = \floor{\dfrac{1+\epsilon}{\kldiv(\mu_i,\mu_{i^*})} \Bigg( \log(n)+3\log(\log(n)) \Bigg)}, \label{eq:KL-Proof-EllPrime-Value}
\end{align}
and note that $\kldiv(\mu_i,\mu_{i^*}) = \kldiv^+(\mu_i,\mu_{i^*})$. Let
$\ell' = 1 + K_n$.
Then we have:
\begin{align}
& \sum_{s=\ell'}^{n} \Event{ s\kldiv^+(\hat{\mu}_{i,s}, \mu_{i^*}) \le \log(n) + 3 \log(\log(n)) }\nonumber\\
& \le  \sum_{s=K_n+1}^{\infty} \Event{ (K_n+1) \kldiv^+(\hat{\mu}_{i,s}, \mu_{i^*})  \le \log(n) + 3 \log(\log(n)) } \nonumber\\
& \le\sum_{s=K_n +1}^{\infty} \Event{ \kldiv^+(\hat{\mu}_{i,s},\mu_{i^*}) < \dfrac{\kldiv(\mu_i, \mu_{i^*})}{1+\epsilon} }. \label{eq:KL-Proof-After-EllPrime-Result}
 \end{align}

Putting the value of $\ell'$ and inequalities \eqref{eq:KL-Proof-EllPrime-Value} and \eqref{eq:KL-Proof-After-EllPrime-Result} back into \eqref{eq:KL-Proof-Lemma7-Result} and combining with \eqref{eq:General-UCB-Modification}, we get
\begin{align*}
\ExpVal{T_i(n)} & \le  \ \dfrac{1+\epsilon}{\kldiv(\mu_i,\mu_{i^*})} \Bigg( \log(n)+3\log(\log(n)) \Bigg) + \ExpVal{ G_{i,n}^* } + 1 +  \\
& \  \sum_{t=1}^{n} \Prob{ B_{i^*,S_{i^*}(t-1),t} < \mu_{i^*} } + \ExpVal{G^*_{i,n}} \sum_{s=K_n+1}^{\infty} \Prob{ \kldiv^+(\hat{\mu}_{i,s},\mu_{i^*}) < \dfrac{\kldiv(\mu_i, \mu_{i^*})}{1+\epsilon} },
\end{align*}
where the last term is a result of the delays being independent of the rewards.
The first summation can be bounded using Theorem \ref{thm:KL-UCB-Stat-Tool}, for which it suffices to set $\epsilon_t = 1, 1 \le t \le n$, and use the sequence of observed rewards $(h'_{i,t})$ for the arm under consideration as the sequence $(Y_t)$ in the theorem. In the same way as the analysis of \citet{Garivier:2011}, this gives an upper bound of the form $C'_1 \log(\log n)$ with the same value of $C'_1 \le 7$ as in the non-delayed setting. The second summation can be bounded by Lemma \ref{lem:KL-8}. Therefore, the expected number of times a suboptimal prediction is made is bounded by:
\begin{align*}
 \ExpVal{T_i(n)} \le & \ \dfrac{1+\epsilon}{\kldiv(\mu_i,\mu_{i^*})} \Bigg( \log(n)+3\log(\log(n)) \Bigg) + 
C'_1 \log(\log(n)) + \dfrac{C_2(\epsilon)}{n^{\beta(\epsilon)}} \ExpVal{G^*_{i,n}} + \ExpVal{G_{i,n}^*} +1.
\end{align*}
Combining this with \eqref{eq:Stochastic-Exp-Regret} and letting $C_1 = C'_1 + 3$ finishes the proof.
\end{proof}

\end{document}